\documentclass{article} 
\usepackage{nips12submit_e,times}
\input{header.tex.input}

\title{Dual Decomposition from the Perspective of\\ Relax, Compensate
  and then Recover}

\author{
Arthur Choi \\
Computer Science Department \\
University of California, Los Angeles\\
Los Angeles, CA 90095 \\
\texttt{aychoi@cs.ucla.edu} \\
\And
Adnan Darwiche \\
Computer Science Department \\
University of California, Los Angeles\\
Los Angeles, CA 90095 \\
\texttt{darwiche@cs.ucla.edu} \\
}

\nipsfinalcopy 

\begin{document}

\maketitle

\begin{abstract}
Relax, Compensate and then Recover (RCR) is a paradigm for approximate
inference in probabilistic graphical models that has previously
provided theoretical and practical insights on iterative belief
propagation and some of its generalizations.  In this paper, we
characterize the technique of dual decomposition in the terms of RCR,
viewing it as a specific way to compensate for relaxed equivalence
constraints.  Among other insights gathered from this perspective, we
propose novel heuristics for recovering relaxed equivalence
constraints with the goal of incrementally tightening dual
decomposition approximations, all the way to reaching exact
solutions. We also show empirically that recovering equivalence
constraints can sometimes tighten the corresponding approximation (and
obtaining exact results), without increasing much the complexity of
inference.
\end{abstract}

\section{Introduction}

Relax, Compensate and then Recover (RCR) is a paradigm for approximate
inference that is based on performing three steps~\cite{ChoiDarwiche11}.  
First, one relaxes equivalence constraints in a given model to obtain a simplified model
that is tractable for exact inference.  Second, one compensates for the relaxed equivalences 
by enforcing a weaker notion of equivalence.  
Finally, by recovering equivalence constraints in a selective way, one can incrementally obtain increasingly
accurate approximations, all the way to exact solutions. This paradigm is flexible enough to
characterize existing algorithms for approximate inference, such as
iterative belief propagation (IBP)~\cite{pearl88b,ChoiDarwiche06a,ChoiDarwiche08c}.  
Moreover, a system based on RCR was also successfully employed in the UAI 2010
evaluation of approximate inference, where it was the leading system
in two of the most time-constrained categories evaluated~\cite{UAIEvaluation10}.  

Dual decomposition is a popular and effective approach for approximating MPE
problems in probabilistic graphical models~\cite{JohnsonMW07,KomodakisPT07,SontagGJ11}.\footnote{MPE refers
to the problem finding a {\em complete instantiation} of a graphical model with
maximal probability. This is commonly referred to as MAP as well. However,
many authors reserve MAP to the problem of finding a {\em partial instantiation}
with a maximal probability, which is a much more difficult task computationally 
than MPE. We observe this distinction between MPE and MAP in this paper.}
This technique has a number of desirable properties.  For example, it provides an upper bound on the
original MPE problem, which in some cases, can be tight.  Moreover,
algorithms for solving the corresponding dual optimization problem
have desirable theoretical properties, such as monotonic improvements as in block coordinate descent algorithms.  

In this paper, we formulate dual decomposition as an instance of RCR.
In particular, we view dual decomposition as a particular way of
restoring a weaker notion of equivalence when one relaxes an
equivalence constraint.  From the viewpoint of RCR, this
perspective gives rise to a new family of compensations with
distinctive properties, such as upper bounds on MPE problems, but also
upper bounds on the partition function.  From the viewpoint of
dual decomposition, this perspective 
(a) gives rise to a new approach to tightening upper bounds, based on new heuristics for recovering
equivalence constraints;
(b) expands the reach of dual decomposition by allowing its application to other inference tasks beyond MPE;
and 
(c) positions dual decomposition to capitalize on the vast literature on exact inference in addition to its
classical capitalization on the optimization literature.

Empirically, we show that the recovery of equivalence
constraints in our RCR formulation of dual decomposition can incrementally and effectively tighten the upper
bounds of dual decomposition, leading to optimal solutions in
some cases while recovering only a few equivalence constraints, and without increasing much the complexity
of inference.

\section{Dual Decomposition} \label{sec:dual}

\def\eqv{{eq}}
\def\lm{{\theta}}

We first illustrate the technique of dual decomposition using a concrete example, deferring the
reader to references such as~\cite{SontagGJ11} for a more general treatment.

Consider the MRF \(\psi(A,B,C) = \psi_1(A,B) \psi_2(B,C) \psi_3(A,C)\), where the goal is to find an instantiation
\(a, b, c\) of variables \(A, B, C\) that maximizes \(\psi(a,b,c)\). We refer to this as the {\bf MPE problem.}  
We also refer to \(\max_{a,b,c} \psi(a,b,c)\) as the {\bf MPE value}  and to the maximizing instantiation \(a, b, c\) as an {\bf MPE instantiation.} 
Finally, an MRF induces the probability distribution \(\pr(A,B,C) = \frac{1}{Z} \psi(A,B,C)\), where
we refer to \(Z = \sum_{a,b,c} \psi(a,b,c)\) as the {\bf partition function.}

Dual decomposition is a technique for approximating the MPE problem, which can be described concretely as follows.
We first clone the occurrence of each variable in each factor, leading to auxiliary variables 
\(A_1, B_1\) and \(B_2, C_2\) and \(A_3, C_3\). We now have the {\em fully decomposed} MRF:
\begin{multline*}
\psi(A,B,C,A_1,B_1,B_2,C_2,A_3,C_3) = 
\psi_1(A_1,B_1) \: \psi_2(B_2,C_2) \: \psi_3(A_3,C_3) \\
\eqv(A,A_1) \: \eqv(A,A_3) \: \eqv(B,B_1) \: \eqv(B,B_2) \: \eqv(C,C_2) \: \eqv(C,C_3),
\end{multline*}
where \(\eqv(X,X_i)\) is an equivalence constraint. That is, \(\eqv(x,x_i)=1\) when \(x=x_i\) and \(\eqv(x,x_i)=0\) when \(x \neq x_i\).
Note that \(\psi(a,b,c,a_1,b_1,b_2,c_2,a_3,c_3) = \psi(a,b,c)\) when \(a=a_1=a_3\), \(b=b_1=b_2\) and \(c=c_2 = c_3\);
otherwise, \(\psi(a,b,c,a_1,b_1,b_2,c_2,a_3,c_3) = 0\). Hence,
\[ \max_{a,b,c} \psi(a,b,c) = \max_{a,b,c,a_1,b_1,b_2,c_2,a_3,c_3} \psi(a,b,c,a_1,b_1,b_2,c_2,a_3,c_3). \]
The original and fully decomposed MRFs are then equivalent as far as computing the MPE value.

We now relax the equivalence constraints (i.e., drop them), while replacing each constraint \(\eqv(X,X_i)\) by
\(\lm_j(X)/\lm_j(X_i)\) (which is equal to one when \(x=x_i\)), leading to:
\begin{multline*}
\psi(A,B,C,A_1,B_1,B_2,C_2,A_3,C_3) = \\
\psi_1(A_1,B_1) \: \psi_2(B_2,C_2) \: \psi_3(A_3,C_3) 
\frac{\lm_1(A)\lm_3(A)}{\lm_1(A_1)\lm_3(A_3)}\frac{\lm_1(B)\lm_2(B)}{\lm_1(B_1)\lm_2(B_2)}\frac{\lm_2(C)\lm_3(C)}{\lm_2(C_2)\lm_3(C_3)}.
\end{multline*}
Note that \(\psi(a,b,c,a_1,b_1,b_2,c_2,a_3,c_3) = \psi(a,b,c)\) when \(a=a_1=a_3\), \(b=b_1=b_2\) and \(c=c_2 = c_3\);
otherwise, \(\psi(a,b,c,a_1,b_1,b_2,c_2,a_3,c_3)\) is incomparable to \(\psi(a,b,c)\). Hence,
\begin{eqnarray*} 
\max_{a,b,c} \psi(a,b,c) 
& \leq & \max_{a,b,c,a_1,b_1,b_2,c_2,a_3,c_3} \psi(a,b,c,a_1,b_1,b_2,c_2,a_3,c_3) \\
& = & 
\left[\max_{a} \lm_1(a)\lm_3(a)\right]
\left[\max_{b} \lm_1(b)\lm_2(b)\right]
\left[\max_{c} \lm_2(c)\lm_3(c)\right] \\
&  & 
\left[\max_{a_1,b_1} \frac{\psi_1(a_1,b_1)}{\lm_1(a_1)\lm_1(b_1)}\right] 
\left[\max_{b_2,c_2} \frac{\psi_2(b_2,c_2)}{\lm_2(b_2)\lm_2(c_2)}\right] 
\left[\max_{a_3,c_3} \frac{\psi_3(a_3,c_3)}{\lm_3(a_3)\lm_3(c_3)}\right]
\end{eqnarray*}
This is called the {\em dual objective} and is guaranteed to provide an upper bound on the MPE value, \(\max_{a,b,c} \psi(a,b,c)\), 
regardless of the specific values chosen for multipliers \(\lm_j(x) > 0\). However, one can improve the upper bound by searching for 
multipliers \(\lm_i(x)\) that minimize the dual objective. 


Minimization problems such as this one can be tackled using techniques
from the optimization literature.  For example, subgradient methods
are applicable to objective functions that are not differentiable,
such as the one above.  They are also guaranteed to minimize the dual
objective to optimality, with appropriate choice of step sizes.  For
another example, block coordinate descent methods monotonically
decrease the dual objective at each step, and can yield faster
convergence rates than subgradient methods.  However, they are not
necessarily guaranteed to minimize the dual objective.  See
\cite{SontagGJ11} for a more thorough introduction to dual
decomposition, and algorithms for the dual optimization problem.

\shrink{
Viewing this optimization problem in terms of linear programming (LP)
relaxations \cite{GlobersonJ07,SontagGJ11}, we can tighten this upper
bound further by tightening the LP relaxation.  In \cite{SontagGJ11},
this takes the form of introducing new linear constraints that
correspond to local constraints on marginal consistency, which in turn
implies a tighter dual objective function.

\begin{figure}[tb]
\begin{center}
 \includegraphics[scale=.4,clip=true,angle=0]{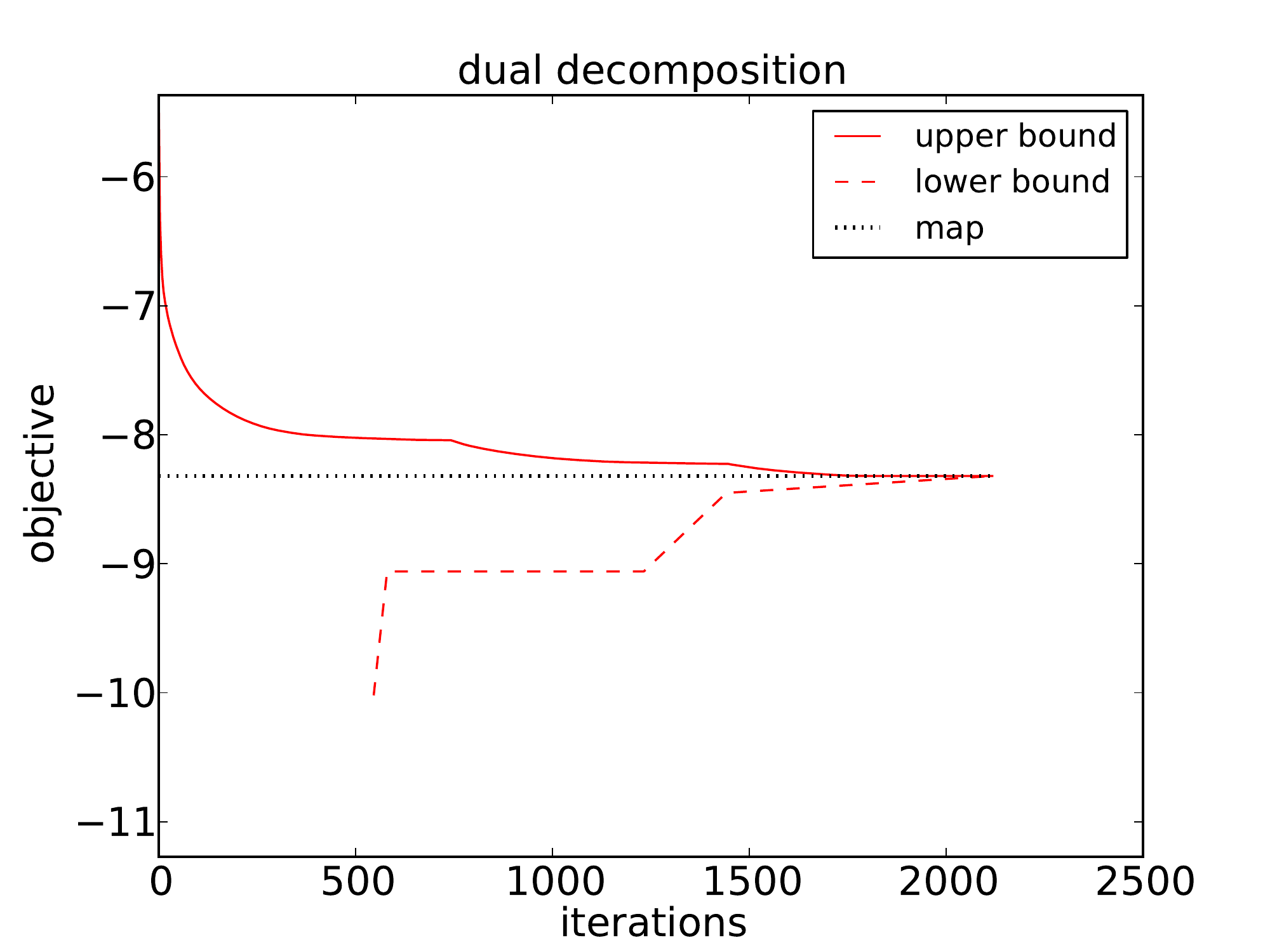}
 \caption{Tightening upper and lower bounds, as we iteratively
   improve our approximation (moving right on the \(x\)-axis).
   \label{fig:dual}}
\end{center}
\end{figure}

If one has an instantiation \(a,b,c,a_1,b_1,c_1,a_2,b_2,c_2\), it can
be evaluated in the original MRF, leading to a lower bound on the
MPE value (since no instantiation can have a value
greater than that of the MPE).  We can obtain such an instantiation,
for example, by decoding one from the solution of the LP relaxation.
By tightening our LP relaxation, we tighten the upper bound provided
by the dual objective.  We can also improve the quality of the
instantiation that we obtain from it, tightening our lower bound as
well.  When these two bounds meet, we then obtain the optimal value of
our MPE problem \cite{SontagGJ11}; see Figure~\ref{fig:dual}.
}

\section{Relax, Compensate, and then Recover} \label{sec:rcr}

RCR is an approximate inference framework, which is based on three steps. 
The first step relaxes equivalence constraints from the original model. The second step compensates
for the relaxed equivalences by enforcing some weaker notion of equivalence. The third step recovers back some of the equivalences
in an anytime fashion, with the goal of improving the approximation. The main computational work performed by RCR is in the compensation step, which
requires exact inference on the relaxed model (any exact inference algorithm can be used for this purpose). The recovery step
may also entail computational work, although this depends largely on the recovery heuristics (some heuristics can be computed
as a side effect of the compensation step, as we show later).

We will next illustrate the three steps of RCR using the same example discussed above. 
For a more general treatment of RCR, however, the reader is referred to~\cite{ChoiDarwiche11}.

\subsection{Relax}
The first step of RCR is similar to the one used by dual decomposition: We clone variables and introduce equivalence constraints, 
leading to the following model:
\begin{multline*}
\psi(A,B,C,A_1,B_1,B_2,C_2,A_3,C_3) = 
\psi_1(A_1,B_1) \: \psi_2(B_2,C_2) \: \psi_3(A_3,C_3) \\
\eqv(A,A_1) \: \eqv(A,A_3) \: \eqv(B,B_1) \: \eqv(B,B_2) \: \eqv(C,C_2) \: \eqv(C,C_3).
\end{multline*}
We can then relax an equivalence constraint by simply dropping it from the model. For example,
relaxing all equivalence constraints leads to the following model, which is {\bf fully decomposed:}
\[
\psi(A,B,C,A_1,B_1,B_2,C_2,A_3,C_3) = 
\psi_1(A_1,B_1) \: \psi_2(B_2,C_2) \: \psi_3(A_3,C_3).
\]

In principle, one can relax as many constraints as one wishes---normally, until the model is disconnected enough to be
feasible for exact inference. RCR, however, typically relaxes enough equivalence constraints to render the model
fully decomposed. It then recovers some of these constraints incrementally and selectively, until it runs out of time
or until the model becomes too connected to be feasible for exact inference. More on this later.

\subsection{Compensate}

Compensating for a relaxed equivalence constraint, say, \(\eqv(A,A_1)\),
is done by adding factors \(\lm_{A_1}(A)\) and \(\lm_{A}(A_1)\) in lieu of factor \(\eqv(A,A_1)\), leading to the {\bf compensated model:}
\begin{multline*}
\psi(A,B,C,A_1,B_1,B_2,C_2,A_3,C_3) = 
\psi_1(A_1,B_1) \: \psi_2(B_2,C_2) \: \psi_3(A_3,C_3) \\
\underline{\lm_{A_1}(A)\lm_{A}(A_1)} \: \eqv(A,A_3) \: \eqv(B,B_1) \: \eqv(B,B_2) \: \eqv(C,C_2) \: \eqv(C,C_3).
\end{multline*}
The added factors, \(\lm_{A_1}(A)\) and \(\lm_{A}(A_1)\), are sometimes called {\em compensation factors.}  Note that we shall omit the subscripts \(X_i\) and \(X\) when it is clear that factors \(\lm(X)\) and \(\lm(X_i)\) refer to the compensation factors for equivalence constraint \(\eqv(X,X_i)\).
Moreover, whenever we refer to a state \(x\) of variable \(X\),
we will denote the corresponding state of variable \(X_i\) by \(x_i\),
unless otherwise stated.

A {\em compensation scheme} is a set of conditions on the values of compensating factors. Each compensation
scheme leads to a class of approximations. In phrasing such conditions, we will write \(\mpe(a)\) 
to denote the {\em MPE marginal,} \(\max_{b,c} \psi(a,b,c)\). 
We will also write \(Z(a)\) to denote the {\em partition function marginal,} \(\sum_{b,c} \psi(a,b,c)\). 

The following is a common condition used by different RCR compensation schemes.
\begin{definition}
A compensation scheme for relaxed equivalence \(\eqv(X,X_i)\) satisfies {\bf pr-equivalence} iff the distribution
induced by the compensated model satisfies \(\pr(x) = \pr(x_i)\) for all values \(x\) and their corresponding values \(x_i\). 
Moreover, it satisfies {\bf mpe-equivalence} iff \(\mpe(x) = \mpe(x_i)\) for all values \(x\) and their corresponding values \(x_i\).
\end{definition}

A common and powerful technique for deriving further conditions on the compensation scheme is based on
considering a single relaxed equivalence, under some idealized situation, and finding out what that idealization 
implies. Suppose, for example, that relaxing the equivalence constraint \(\eqv(X,X_i)\) 
splits the model into two disconnected components, one containing variable \(X\) and another containing variable \(X_i\). 
This idealized situation implies the following condition, which is the only condition that leads 
to exact node marginals.

\begin{definition}
A compensation scheme for relaxed equivalence \(\eqv(X,X_i)\) satisfies {\bf model-split}  iff the distribution
induced by the compensated model satisfies pr-equivalence and 
\[
\pr(x) = \frac{\lm(x)\lm(x_i)}{\sum_{x} \lm(x)\lm(x_i)}.
\]
\end{definition}
On fully decomposed models, this compensation scheme leads to IBP approximations ~\cite{ChoiDarwiche06a,ChoiDarwiche09b}, and further the Bethe free energy approximation of the partition function~\cite{Yedidia05,ChoiDarwiche08c}.

\subsection{Finding compensations} \label{sec:find compensations}

The main computational work performed by RCR is in finding compensations that satisfy some stated conditions. This is usually
done by deriving a characterization of the compensation, which yields fixed-point iterative equations. 
For example, compensations that satisfy model-split have been characterized as follows \cite{ChoiDarwiche06a}.
\begin{theorem}\label{theo:edbp}
A compensation scheme for relaxed equivalence constraint \(\eqv(X,X_i)\) satisfies model-split iff the partition function \(Z\) of
the compensated model satisfies
\begin{align}
\theta(x) & = \alpha \frac{Z(x_i)}{\theta(x_i)} &
\theta(x_i) & = \alpha \frac{Z(x)}{\theta(x)}
\label{eq:edbp}
\end{align}
for all states \(x,\) and their corresponding states \(x_i\).
Here, \(\alpha \) is an arbitrary normalizing constant.
\end{theorem}
This theorem identifies update equations which form the basis of an iterative fixed-point algorithm 
that searches for model-split compensations.\footnote{The required quantities correspond to partial derivatives, which can be computed efficiently in traditional frameworks 
for inference \cite{darwicheJACM-POLY,ParkAIJ04}.} 
In fact, the message-passing updates of IBP are precisely the fixed-point iterative updates implied by Equation~\ref{eq:edbp}~\cite{ChoiDarwiche06a}.

\subsection{Recover}

RCR typically relaxes enough equivalence constraints to yield a fully decomposed
model. It then recovers equivalence constraints incrementally and selectively,
until it runs out of time or the model becomes too connected to be feasible for exact
inference. The recovery process is based on a heuristic, called a {\em recovery heuristic,}
that tries to identify the constraints whose relaxation has been most damaging to
the quality of an approximation.

A number of recovery heuristics have been proposed previously. One of these
heuristics is based on mutual information~\cite{ChoiDarwiche06a} and is designed
for the use with the compensation scheme that satisfies model-split.  Another
heuristic was used by RCR at the UAI'10 approximate inference 
evaluation~\cite{UAIEvaluation10,ChoiDarwiche11}, 
which was critical to the performance (and success) of RCR in that evaluation.

Combining recovery, with compensations that satisfy model-split, yields approximations that correspond to iterative joingraph propagation
(IJGP) approximations~\cite{AjiMcEliece01,DechterKM02,ChoiDarwiche06a}.\footnote{Similar characterizations 
and generalizations of IBP have been shown in~\cite{MinkaPhd,WainwrightJW03,YedidiaFreemanWeiss03}.}

\section{A New Compensation Scheme: Dual Decomposition} 

We will now consider a new compensation scheme for RCR, which gives rise to dual decomposition approximations 
of Section~\ref{sec:dual} when the inference task of RCR is that of computing MPE.


We start with the following family of compensation schemes.
\begin{definition}
A compensation scheme for relaxed equivalence \(\eqv(X,X_i)\) satisfies {\bf upper-bound} iff
\begin{equation}
\theta(x)\theta(x_i) = 1, \mbox{  for all values \(x\) and their corresponding values \(x_i\)}.
\label{eq:cs-upper-bound}
\end{equation}
\end{definition}
The above condition leads to the following interesting guarantee.
\begin{theorem} \label{thm:bound}
A compensation scheme that satisfies upper-bound leads to a compensated model whose partition function
is an upper bound on the exact partition function, and whose MPE value is an upper bound on the exact MPE value.
\end{theorem}

Combining the upper-bound condition with pr/mpe-equivalence  leads to a compensation scheme
that characterizes and generalizes dual decomposition approximations, as we show next.

\begin{definition}\label{def:cs-dd}
A compensation scheme satisfies {\bf pr-dd} iff it satisfies upper-bound and pr-equivalence.
Moreover, it satisfies {\bf mpe-dd} iff it satisfies upper-bound and mpe-equivalence.
\end{definition}

The following theorem provides a characterization of the pr-dd and mpe-dd compensation schemes, which
can be used to search for compensations in fully decomposed models.

\begin{theorem} \label{thm:dd-update}
For a single equivalence constraint \(\eqv(X,X_i)\), a compensation
scheme satisfies pr-dd iff for all values \(x,\) and their
corresponding values \(x_i,\) the compensated model satisfies
\begin{equation} \label{eq:dd-update}
\theta(x) 
= \left( 
\frac{Z(x_i)/\theta(x_i)}
     {Z(x)/\theta(x)}
\right)^{\frac{1}{2}}
\quad\quad
\theta(x_i) 
= \left( 
\frac{Z(x)/\theta(x)}
     {Z(x_i)/\theta(x_i)}
\right)^{\frac{1}{2}}
\end{equation}
The scheme satisfies mpe-dd iff it satisfies the above condition with \(\mpe(.)\) substituted for \(Z(.)\).
\end{theorem}
There is one subtlety about the above theorem, in comparison to Theorem~\ref{theo:edbp}. The equation given
in this theorem can be used as an update equation only when variables \(X\) and \(X_i\) are independent in
the compensated model (otherwise, the left-hand side will depend on the right-hand side). When the compensated model
is fully decomposed, this condition is met (after taking into account the division of the compensating factors from the partition function marginals). More generally, when relaxing the equivalence constraint
\(\eqv(X,X_i)\) splits the model into two disconnected components, one containing \(X\) and the other containing \(X_i\),
the condition is also met.

In fully decomposed models, one can use the above update equation to search for compensations that satisfy pr-dd or mpe-dd,
in the same way that Equation~\ref{eq:edbp} can be used to search for compensations that satisfy
model-split (see Section~\ref{sec:find compensations}). We actually have a stronger result.

\begin{theorem} \label{thm:dd}
When the compensated model is fully decomposed, the fixed-point iterative updates of Equation~\ref{eq:dd-update}
correspond precisely to the block coordinate descent updates of the sum-product and max-sum diffusion algorithms,
respectively.
\end{theorem}

This theorem has the following main implication: When computing MPE using RCR with an mpe-dd compensation
scheme, one obtains approximations that correspond precisely to those computed by the dual decomposition
technique of Section~\ref{sec:dual} (assuming a fully decomposed model). 
In particular, the MPE computed using RCR corresponds precisely to one computed
at a fixed-point of a block coordinate descent algorithm such as max-sum diffusion \cite{JohnsonMW07,KomodakisPT07,SontagGJ11}.


We finally point out that the fixed-point iterative algorithm suggested by
Equation~\ref{eq:dd-update} also inherits properties that make block
coordinate descent algorithms so popular, such as monotonic
improvements of the approximation (i.e., MPE value or partition function), 
when equivalence constraints are updated one at a time \cite{GlobersonJ07}.

\section{New Recovery Heuristics for Dual Decomposition} 

Our main result thus far is that the dual decomposition technique for computing MPE corresponds
to an instance of RCR in which (a) enough equivalence constraints are relaxed to yield a fully
decomposed model and (b) the relaxed equivalences are compensated using the mpe-dd condition. 

This, however, corresponds to the degenerate case of RCR. One can obtain much better approximations by recovering
some of the relaxed equivalence constraints, which can be done incrementally and selectively. In the general RCR framework,
this recovery process usually continues until one runs out of time or until the model is too connected to be accessible to exact inference (which
is needed to search for compensations). As we show in the next section, however, this process can actually terminate much earlier,
as we may be able to detect when the computed MPE is exact.

In this section, however, we will focus our attention on two tasks. First, we design heuristics for recovering equivalence constraints
in the context of pr-dd and mpe-dd compensation scheme. Second, we identify a more general update equation than the one
of Theorem~\ref{thm:dd-update}, which, as mentioned earlier, is only applicable in restricted settings. 
Such an update equation
is necessary if we were to search for compensations in a model that is not fully decomposed.

\begin{theorem} \label{thm:dd-update2}
For a single equivalence constraint \(\eqv(X,X_i)\), with binary variables \(X\) and \(X_i\), 
a compensation scheme satisfies pr-dd iff the compensated model satisfies
\begin{equation} \label{eq:dd-update2}
\frac{\theta(x)}{\theta(\n(x))}
= \left( 
\frac{Z(\n(x),x_i)/\theta(\n(x))\theta(x_i)}
     {Z(x,\n(x)_i)/\theta(x)\theta(\n(x)_i)}
\right)^{\frac{1}{2}}
\qquad\qquad
\frac{\theta(x_i)}{\theta(\n(x)_i)}
= \left( 
\frac{Z(x,\n(x)_i)/\theta(x)\theta(\n(x)_i)}
     {Z(\n(x),x_i)/\theta(\n(x))\theta(x_i)}
\right)^{\frac{1}{2}}
\end{equation}
The scheme satisfies mpe-dd iff it satisfies the above condition with \(\mpe(.)\) substituted for \(Z(.)\).
\end{theorem}
There are two differences between Equation~\ref{eq:dd-update2} and the earlier Equation~\ref{eq:dd-update}.
First, the new equation is applicable even when variables \(X\) and \(X_i\) are not independent in the compensated model.
Hence, we can use this equation to implement a fixed-point iterative algorithm that searches for
compensations in any model.\footnote{In our implementation, we simply set
  \(\theta(\n(x)) = \theta(\n(x)_i) = 1\).}
Second, the new equation is restricted to binary variables as we have yet to derive a version of this for multi-valued
variables.
Similar to Equation~\ref{eq:dd-update}, however, the new equation monotonically improves the approximation, 
when equivalence constraints are updated one at a time.

We now turn our attention to recovery heuristics. Our first observation is as follows: One can efficiently compute
the exact effect of recovering a single equivalence constraint on the quality of an approximation (i.e., partition function
or MPE value). In particular, the
improvement due to recovering a single equivalence constraint can be computed as a side effect of
the fixed-point update by Equation~\ref{eq:dd-update2}.\footnote{The partition function after
recovering a single constraint \(\eqv(X,X_i)\) is \(\frac{Z(x,x_i)}{\theta(x)\theta(x_i)} + \frac{Z(\n(x),\n(x)_i)}{\theta(\n(x))\theta(\n(x)_i)}\)~\cite{ChoiDarwiche08c}.
Moreover, the MPE value after recovering the constraint is \(\max\{\frac{\mpe(x,x_i)}{\theta(x)\theta(x_i)}, \frac{\mpe(\n(x),\n(x)_i)}{\theta(\n(x))\theta(\n(x)_i)}\}.\)}
Thus, our first recovery heuristic imposes no additional overhead as we can compute the exact impact of
recovering each equivalence constraint during the compensation phase.\footnote{Note, however that subsequent fixed-point updates
for other equivalence constraints will in principle invalidate the measured impacts of previous constraints.  On the other hand,
computing this impact requires computations that would allow us to perform an update anyways.}

This first heuristic, however, may not distinguish each equivalence constraint sufficiently (many constraints
may have the same impact upon recovery).  Thus, we propose a secondary recovery heuristic which is specific to mpe-dd and motivated as follows.
Given a current model, suppose that the recovered MPE instantiation is \(\x\) and has value \(m\). In general,
\(m\) is only an upper bound on the exact MPE value as instantiation \(\x\) may violate some
relaxed equivalence constraints, \(\eqv(X,X_i)\)---that is, instantiation \(\x\) may set \(X\) and \(X_i\) to different values. 
However, if instantiation \(\x\) does not violate any of the relaxed equivalence constraints, then \(m\) must
be the exact MPE value. Our secondary recovery heuristic will therefore recover those equivalence 
constraints that are currently violated by the instantiation \(\x\). By recovering such equivalence
constraints, we hope to reduce the number of violated equivalence
constraints in our approximate MPE instantiation, and thus hope to
recover an exact MPE instantiation; cf. reducing the duality gap as in
\cite{SontagMGJW08}. 

Consider, in contrast, the ``recovery'' heuristic suggested by
\cite{SontagMGJW08}, which introduced local consistency
constraints to tighten a linear programming (LP) relaxation that corresponds
to the dual objective of dual decomposition.  This heuristic sought
to tighten an outer bound on the marginal polytope, which would
normally require exponentially many linear constraints in an LP that would exactly solve an MPE problem.  The ``recovery'' heuristic suggested by
\cite{SontagMGJW08}, introduces local consistency constraints
over triplet clusters, which was particularly effective at solving
challenging classes of MPE problems, such as protein design problems
\cite{YanoverMW06}.  However, introducing triplet constraints by
themselves may not be sufficient to completely tighten the dual bound,
and otherwise, there are exponentially many local consistency
constraints available to choose from.  In contrast, the RCR recovery process
yields an incremental and full spectrum of approximations, leading up to exact inference when all equivalence constraints have been recovered.
Thus, we view RCR recovery as a complementary
approach to the techniques of \cite{SontagMGJW08}, when triplet
constraints are not sufficient to extract the exact MPE solution.

\section{An Empirical Perspective}

We evaluate our new recovery heuristics based on their ability to
extract an exact MPE solution for a given probabilistic graphical
model.  In our {\em first set of experiments,} our goal is to illustrate that
RCR can obtain an exact MPE solution by recovering equivalence
constraints, without impacting much the complexity of inference.  For
our second set of experiments, we compared RCR with MPLP in their
ability to find exact MPE solutions based on their respective
approaches to tightening a relaxation, which is by adding triplet
clusters in the case of MPLP \cite{SontagMGJW08}.\footnote{A public
version of MPLP is available at \url{http://cs.nyu.edu/~dsontag/}.
In our {\em second set of experiments,} we used an updated implementation
of MPLP that was provided to us by the authors of~\cite{SontagMGJW08}.}  
Our goal here is to illustrate that recovering equivalence constraints can also be a viable option for
models where introducing triplet clusters alone is not sufficient to
tighten the dual objective of dual decomposition.

For RCR, starting with a fully decomposed model, we
iteratively recover \(5\) equivalence constraints at a time, as
described in the previous section.  For MPLP, we used the default
settings, which introduced 5 triplet clusters at a time.  RCR was set,
as MPLP was, to run for at most 1000 iterations, before recovering
equivalence constraints and introducing triplet clusters.

As the RCR approach requires only a black-box inference engine to
execute its compensation phase (which requires only marginals, or alternatively, partial
derivatives), we can take advantage of state-of-the-art systems for
exact inference.  This includes advanced approaches for inference
based on arithmetic circuits (ACs), which can effectively exploit
local structure
\cite{Chavira.Darwiche.Sat.2006,Chavira.Darwiche.Ijcai.2007}.  We use
such an inference engine for our experiments, although the benchmarks
that we considered do not necessarily have much local structure.
Using arithmetic circuits, we can also more efficiently compute quantities
such as \(\mpe(x,x_i)/\theta(x)\theta(x_i)\) via lazy evaluation in an arithmetic circuit~\cite{ChoiStandleyDarwiche09}.

We first performed experiments on 50 randomly parameterized grid models,
which we generated using MPLP with default parameters, but assuming
binary variables.  The resulting \(10 \times 10\) grids corresponded
to pairwise MRFs with mixed attractive and repulsive couplings.  The
following table summarizes the number of equivalence constraints (out
of 360 relaxed) that needed to be recovered for RCR to obtain an
optimal MPE solution, and the corresponding complexity of inference
(on average).  Note that the complexity of inference using arithmetic
circuits is linear in the size of the AC, i.e., the number of nodes
and edges in the resulting circuit.  

\shrink{
\begin{center}
\begin{tabular}{rrr}
constraints recovered & \% instances & \% increase in AC size\\ \hline
0--90\hphantom{0} &  0\% & ---\\
91--120   &  4\% &  88.11\% \\
121--150  & 16\% &  93.58\% \\
151--180  & 12\% &  89.31\% \\
181--210  & 18\% & 103.17\% \\
211--240  & 24\% & 100.43\% \\
241--270  & 12\% & 113.78\% \\
271--300  &  6\% & 195.41\% \\
301--330  &  8\% & 308.39\% \\
331--360  &  0\% & ---
\end{tabular}
\end{center}
}
{\small
\setlength{\tabcolsep}{3pt}
\begin{center}
\begin{tabular}{r|r|r|r|r|r|r|r|r}
edges recovered     & 91--120 & 121--150 & 151--180 & 181--210 & 211--240 & 241--270 & 271--300 & 301--330 \\\hline
\% instances        &   4\% & 16\% & 12\% & 18\% & 24\% & 12\% & 6\% & 8\%\\
\% increase in AC size & 88.11\% & 93.58\% & 89.31\% & 103.17\% & 100.43\% & 113.78\% & 195.41\% & 308.39\%
\end{tabular}
\end{center}
}

Observe that RCR was able to recover up to \(240\) equivalence
constraints, and solve \(74\%\) of all MPE problems, without
increasing much---even decreasing in many cases---the complexity of
inference.  Note that we start with a fully decomposed
approximation, and it is easily possible to recover many equivalence
constraints without impacting much the treewidth of a model (it is
possible to recover 200 and only obtain a spanning tree).  Moreover,
AC size can decrease since there are fewer compensating factors to
maintain.  MPLP is also effective on this benchmark, where it can
introduce square clusters into its relaxation \cite{SontagMGJW08},
although such a technique is restricted to grids.

We next performed experiments on Bayesian networks induced from
haplotype data (over 201 binary variables), which are networks with
bounded treewidth \cite{ElidanGould:JMLR08}.  These networks do not
necessarily have as regular a structure that can suggest a natural way
of introducing clusters, such as in grids.  Moreover, note that
triplet clusters alone may not be sufficient to tighten the dual
objective, i.e., to close the duality gap.  In these benchmarks, there
were 69 models, of which 13 models were cases where MPLP failed to
find the optimal MPE solution, given 1000 attempts to tighten its
relaxation (i.e., to introduce local consistency constraints).  In
contrast, RCR was able to obtain the optimal MPE solution in all
cases, after recovering a small number of equivalence constraints.

\begin{figure}[tb]
\begin{center}
  \includegraphics[scale=.32,clip=true,angle=0]{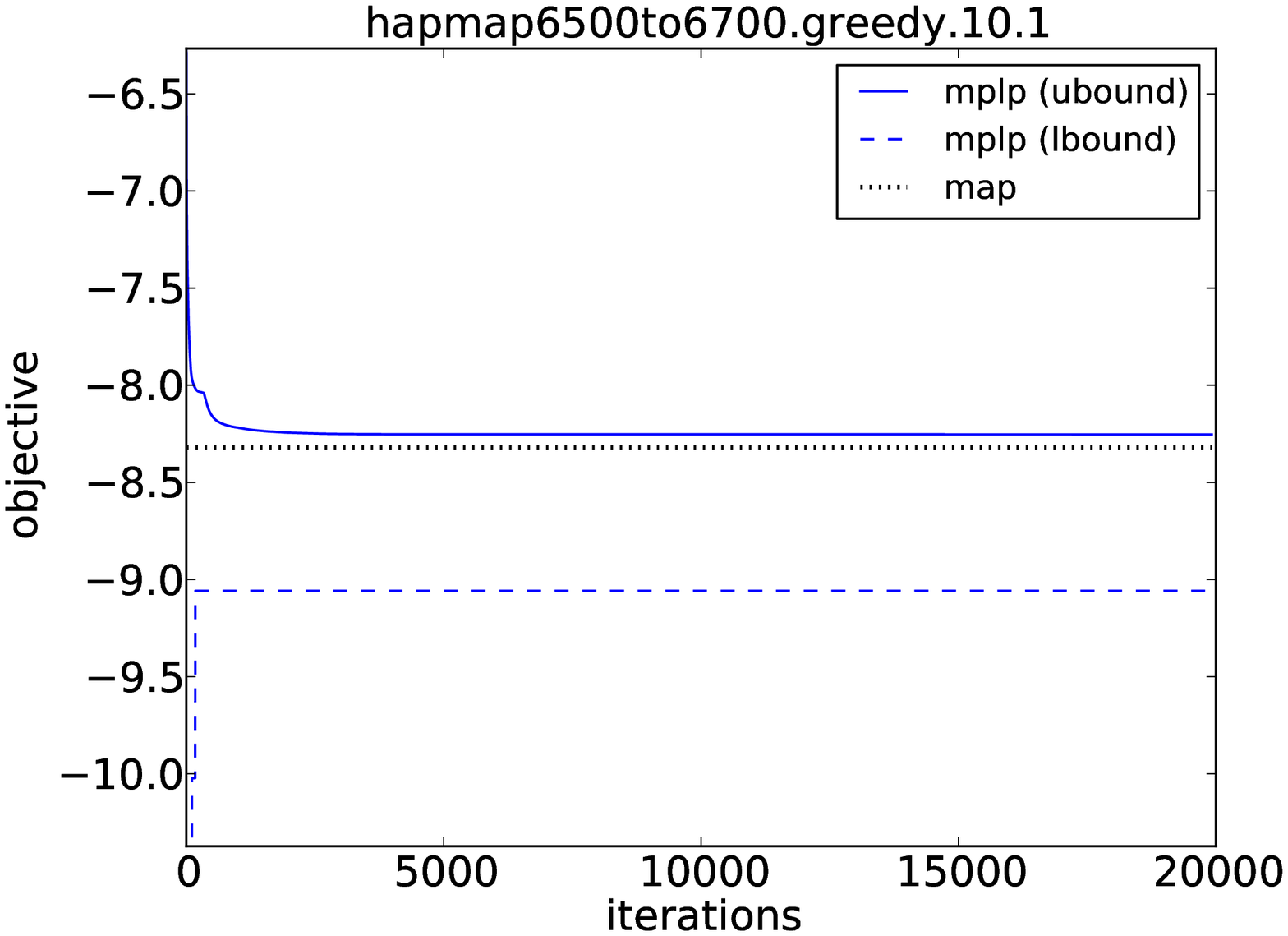}
  \quad
  \includegraphics[scale=.32,clip=true,angle=0]{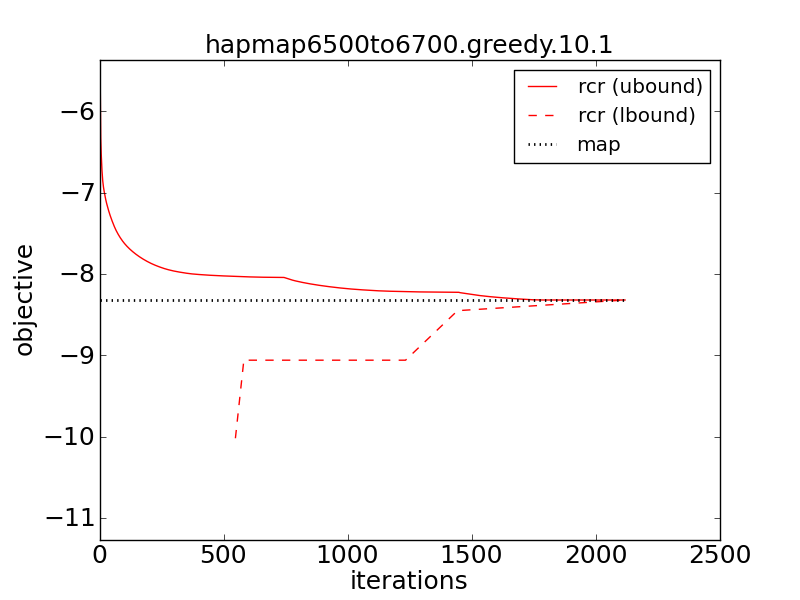}
  \caption{Recovering triplet clusters and equivalence constraints in
    MPLP (left) and RCR (right).  Solid lines indicate the value of
    the dual objective (upper bound), dashed lines indicate the value
    of the current best assignment (lower bound), and the dotted line
    denotes the optimal MPE solution.
    \label{fig:ex1}}
\end{center}
\end{figure}

Figure~\ref{fig:ex1} illustrates an example run of both MPLP and RCR,
in a model where MPLP failed to find an optimal MPE solution.  For the
case of MPLP, one observes that MPLP starts to tighten the gap between
its upper and lower bounds, but fails to tighten it further after some
number of iterations.  In fact, for this particular model, MPLP fails
to find triplet clusters to introduce into its relaxation.  On the
other hand, RCR obtains the optimal solution after recovering only 70
of 451 equivalence constraints.  When we look at the arithmetic
circuits used to do inference in our simplified model, the size
goes down from 38555 to 36729 nodes and edges after
recovering 70 equivalence constraints.

In the following table, we summarized the number of recovered
equivalence constraints needed to obtain an optimal solution, and
the complexity of inference, for the two cases:
\begin{center}
\begin{tabular}{c|c|c|c}
& \# of models & avg. \% recovered & avg. \% increase in AC size \\ \hline
MPLP did not solve  & 13 & 26.93\% & 124.97\% \\
RCR and MPLP solved & 56 & 3.56\% & 99.65\% 
\end{tabular}
\end{center}
In the models that were left unsolved by MPLP, RCR was able to find
exact MPE solutions by recovering only a quarter of the relaxed
equivalence constraints, on average.  This came with only a modest
increase in the complexity of inference, i.e., AC size.  In the models
solved by both MPLP and RCR, very few equivalence constraints needed
to be recovered on average, and in fact led to a very slight decrease
in the complexity of inference.

We finally remark that the second set of experiments involved models
that are not necessarily well suited for recovering triplet clusters
with MPLP.  Moreover, our comparisons with RCR were limited since we
were restricted to models over binary variables (as recovery requires
the use of a compensation algorithm like the one implied by
Theorem~\ref{thm:dd-update2}, which is specific to binary variables).
We plan more thorough empirical comparisons in future work.

\section{Conclusion} \label{sec:conclusion}

In this paper, we formulated the technique of dual decomposition in
the terms of Relax, Compensate and then Recover (RCR).  By formulating
dual decomposition in the more general terms of RCR, we have broadened the
scope of the technique by (a) proposing new recovery heuristics for
tightening the dual objective of dual decomposition, (b) extending it
to other inference tasks, such as bounding the partition function (although this was not evaluated here), 
and (c) formulating it in terms that allows it to easily take advantage of
the vast literature on exact inference, for the purposes of more
effective approximate inference.  Empirically, we showed how these new
recovery heuristics can sometimes be used to obtain exact solutions to MPE
problems, without increasing much the complexity of inference---in particular, on
problems which existing systems based on dual decomposition are not as well suited for.

\subsubsection*{Acknowledgments}

This work has been partially supported by NSF grant \#IIS-1118122.

{\small
\bibliographystyle{unsrt}
\bibliography{bib/dd-rcr,bib/CHANGE,bib/liftedrcr}
}

\newcommand{\art}[1]{\textbf{(*** #1 ***)}}

\def\eql(#1,#2){{#1\!=\!#2}}

\appendix
\section{Proofs}

We first review and refine some notation and some definitions, for the
purposes of our proofs.  Here, variables are denoted by upper case
letters (\(X\)) and their instantiations by lower case letters
(\(x\)). Moreover, sets of variables are denoted by bold upper case
letters (\(\X\)) and their instantiations by bold lower case letters
(\(\x\)).

An MRF \(\psi(\X)\), over a set of variables \(\X\), is a product of
factors \(\psi_i\), which induces a probability distribution
\(\pr(\X)\):
\[
\psi(\X) = \prod_i \psi_i(\X_i)
\qquad
\qquad
\qquad
\pr(\X) = \frac{1}{Z} \psi(\X).
\]
Here, each factor \(\psi_i(\X_i)\) is a function mapping an
instantiation \(\x_i\) of variables \(\X_i\), to a non-negative real
number.  Moreover, \(Z = \sum_{\x} \psi(\x)\) is a normalizing
constant called the partition function.

We are interested in approximations to the partition function, and the
most probable explanation (MPE):
\[
\mpe = \max_{\x} \psi(\x)
\]
We refer to \(\mpe\) as the MPE value, and a maximizing \(\x\) as an
MPE instantiation.  We are also interested in MPE marginals
\(\mpe(x)\) and partition function marginals \(Z(x)\):
\[
\mpe(x) = \max_{\x \models x} \psi(\x)
\qquad
\qquad
\qquad
Z(x) = \sum_{\x \models x} \psi(\x)
\]
where \(\mpe(x)\) can be interpreted as the MPE value of our model,
assuming variable \(X\) takes on the value \(x\); similarly for
partition function marginals.  

We may augment an MRF so that it contains factors \(\eqv(X,Y)\) that
represent equivalence constraints \(X \equiv Y\) between pairs of
variables \(X\) and \(Y\) in \(\X\).  For the purposes of this paper,
we will assume that equivalence constraints arise by cloning a
variable \(X\) that appears in a factor \(\psi_i(\X_i)\) (although our
results hold for equivalence constraints in general).  We will denote
this clone by \(X_i\), and assume an equivalence constraint
\(\eqv(X,X_i)\).  We continue to denote the set of original variables
by \(\X\), but we now denote the set of clone variables by \(\X^c\).
Our MRF with equivalence constraints is thus:
\[
\psi(\X,\X^c) = \prod_i \psi_i(\X^c_i) \cdot \prod_{X \equiv X_i} \eqv(X,X_i)
\]
Note that the distribution and the MPE problem (over the original
variables \(\X\)), as well as the partition function, are all
invariant to the introduction of equivalence constraints, as described
above.  Moreover, whenever we refer to a state \(x\) of variable
\(X\), we will denote the corresponding state of the clone \(X_i\) by
\(x_i\), unless otherwise stated.

We can relax an equivalence constraint \(\eqv(X,X_i)\) by removing its
factor from the MRF, and then compensate for the relaxation by
introducing two unit factors \(\theta(X)\) and \(\theta(X_i)\).  Doing
so, for all equivalence constraints, we obtain a simpler MRF and
distribution
\[
\psi^\prime(\X,\X^c)
= \prod_i \psi_i(\X^c_i) \cdot \prod_{X \equiv X_i} \theta(X) \theta(X_i)
\qquad
\qquad
\qquad
\prp(\X) = \frac{1}{Z^\prime} \psi^\prime(\X)
\]
where \(Z^\prime\) is the corresponding partition function.  Note that
each constraint \(\eqv(X,X_i)\) is associated with unique factors
\(\theta(X)\) and \(\theta(X_i)\), which we may sometimes distinguish
by \(\theta_{X_i}(X)\) and \(\theta_X(X_i)\).

\setcounter{theorem}{0}
\setcounter{equation}{0}

\begin{theorem}
A compensation scheme for relaxed equivalence constraint \(\eqv(X,X_i)\) satisfies model-split iff the partition function \(Z^\prime\) of
the compensated model satisfies
\begin{align}
\theta(x) & = \alpha \frac{Z^\prime(x_i)}{\theta(x_i)} &
\theta(x_i) & = \alpha \frac{Z^\prime(x)}{\theta(x)}
\end{align}
for all states \(x,\) and their corresponding states \(x_i\).
Here, \(\alpha \) is an arbitrary normalizing constant.
\end{theorem}

\begin{proof} 
See \cite{ChoiDarwiche06a}.
\end{proof}

\begin{theorem} 
A compensation scheme that satisfies upper-bound leads to a
compensated model whose partition function is an upper bound on the
exact partition function, and whose MPE value is an upper bound on the
exact MPE value.
\end{theorem}

\begin{proof} 
Consider an equivalence constraint \(\eqv(X,X_i)\).  If variable \(X\)
is set to the value \(x\), and its clone \(X_i\) is set to the
corresponding value \(x_i\), then \(\eqv(x,x_i) = 1 = \theta(x)
\theta(x_i)\) for a compensation satisfying upper-bound.  When \(x \ne
x_i\), we have \(\eqv(x,x_i) = 0 \le \theta(x) \theta(x_i).\)
Moreover, \(\psi(\x) = \psi^\prime(\x,\x^c)\) if instantiation
\(\x,\x^c\) satisfies all equivalence constraints, and \(\psi(\x) = 0
\le \psi^\prime(\x,\x^c)\) when instantiation \(\x,\x^c\) does not.
Thus, \(0 \le \psi(\x) \le \psi^\prime(\x,\x^c)\) for all
instantiations \(\x\) and \(\x^c\).

The MPE of a compensated model is thus an upper bound on the MPE of
the original:
\[
\max_{\x} \psi(\x) 
=   \max_{\x,\x^c : X \equiv X_i} \psi^\prime(\x,\x^c) 
\le \max_{\x,\x^c} \psi^\prime(\x,\x^c).
\]
Here the second maximization is constrained to assignments
\(\x,\x^c\) that satisfy all equivalence constraints
\(\eqv(X,X_i)\).  Similarly, for the partition function:
\[
Z = \sum_{\x} \psi(\x)
  = \sum_{\x,\x^c : X \equiv X_i} \psi^\prime(\x,\x^c)
\le \sum_{\x,\x^c} \psi^\prime(\x,\x^c)
= Z^\prime.
\]
\end{proof}

\setcounter{equation}{2}

\begin{theorem} 
For a single equivalence constraint \(\eqv(X,X_i)\), a compensation
scheme satisfies pr-dd iff for all values \(x,\) and their
corresponding values \(x_i,\) the compensated model satisfies
\begin{equation} 
\theta(x) 
= \left( 
\frac{Z^\prime(x_i)/\theta(x_i)}
     {Z^\prime(x)/\theta(x)}
\right)^{\frac{1}{2}}
\quad\quad
\theta(x_i) 
= \left( 
\frac{Z^\prime(x)/\theta(x)}
     {Z^\prime(x_i)/\theta(x_i)}
\right)^{\frac{1}{2}}
\end{equation}
The scheme satisfies mpe-dd iff it satisfies the above condition with \(\mpe^\prime(.)\) substituted for \(Z^\prime(.)\).
\end{theorem}

\begin{proof}
From the definition of a pr-equivalence, we first have:
\[
\prp(x) 
= \frac{1}{Z^\prime} \frac{\partial Z^\prime}{\partial \theta(x)} \theta(x)
= \frac{1}{Z^\prime} \frac{\partial Z^\prime}{\partial \theta(x_i)} \theta(x_i)
= \prp(x_i)
\]
for all values \(x\), and \(x_i\) respectively.  For a compensation
satisfying upper-bound, we can substitute \(\theta(x_i) =
\frac{1}{\theta(x)}\) and solve for \(\theta(x)\), giving us
fixed-point conditions:
\[
\theta(x) = \left( 
\frac{\partial Z^\prime/\partial \theta(x_i)}
     {\partial Z^\prime/\partial \theta(x)}
\right)^{\frac{1}{2}}
\]
We further remark that \(\frac{\partial Z^\prime}{\partial
  \theta(x)}\) is independent of the unit factor \(\theta(x)\)
since the partition function \(Z^\prime\) is linear in \(\theta(x)\).
Moreover, we can compute \(\frac{\partial Z^\prime}{\partial
  \theta(x)}\) by \(\frac{Z^\prime(x)}{\theta(x)}\), when
\(\theta(x)\) is positive.  Otherwise, partial derivatives can be
computed efficiently in traditional frameworks for inference, as in
\cite{darwicheJACM-POLY,ParkAIJ04}.

The derivation is analogous for MPE, starting from the definition of
mpe-equivalence.
\end{proof}

\begin{theorem} 
When the compensated model is fully decomposed, the fixed-point iterative updates of Equation~\ref{eq:dd-update}
correspond precisely to the block coordinate descent updates of the sum-product and max-sum diffusion algorithms,
respectively.
\end{theorem}

\begin{proof}[Proof of Theorem~\ref{thm:dd}]
Consider an MRF found by taking each factor \(\psi_i(\X_i)\) and
each variable \(X \in \X_i\), and then:
\begin{enumerate}
\item replace variable \(X\) with a
unique clone variable \(X_i\), and 
\item introduce an equivalence constraint \(\eqv(X,X_i)\).
\end{enumerate}
When we relax all equivalence constraints, the resulting model is
fully decomposed, where all of the factors \(\psi_i(\X^c_i)\), now
over clone variables \(\X^c_i\), are disconnected.  We add
compensating factors \(\theta_{X_i}(X)\) and \(\theta_X(X_i)\), where
\(X\) denotes the original variable and \(X_i\) for each equivalence
constraint \(\eqv(X,X_i)\) relaxed.  The resulting MRF, over original
variables \(\X\) and clone variables \(\X^c\) is:
\begin{align*}
\psi^\prime(\X,\X^c)
& = \Big[ \prod_i \psi_i(\X^c_i) \Big] \cdot \Big[ \prod_{X \equiv X_i} \theta_{X_i}(X) \theta_X(X_i) \Big] \\
& = \prod_i \Big[ \psi_i(\X^c_i) \prod_{X_i \in \X^c_i} \theta_X(X_i) \Big] \cdot
    \Big[ \prod_X \prod_{i : X \in \X_i} \theta_{X_i}(X) \Big].
\end{align*}
Note that each factor \(\psi_i(\X^c_i)\) is now associated with a unit
factor \(\theta_X(X_i)\) for each equivalence constraint
\(\eqv(X,X_i)\) that the factor was involved in: one for each \(X_i
\in \X^c_i\).  Each variable \(X\) is associated with a unit factor
\(\theta_{X_i}(X)\) for each equivalence constraint \(\eqv(X,X_i)\)
that variable \(X\) was involved in: one for each factor
\(\psi_i(\X_i),\) where \(X \in \X_i\).
Figure~\ref{fig:decomposition} highlights a decomposition for a simple
MRF.

\begin{figure}[tb]
\begin{center}
 \includegraphics[scale=.55,clip=true,angle=0]{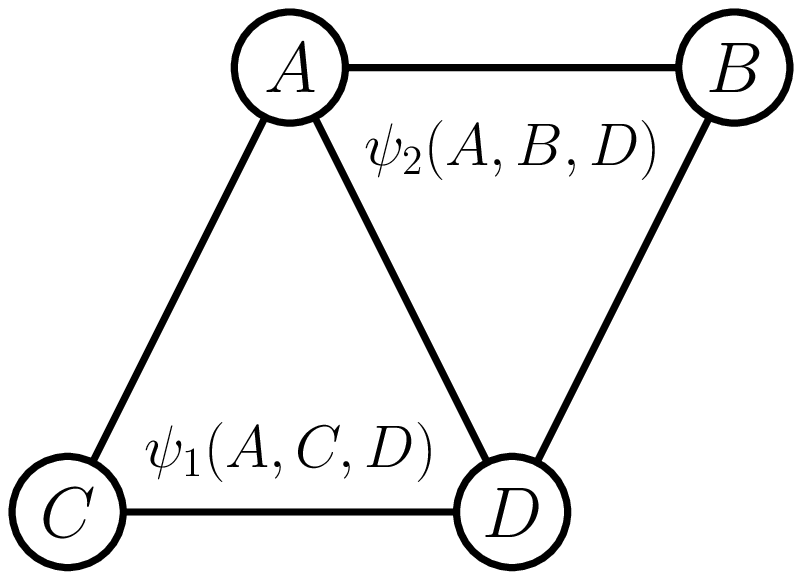}
 \includegraphics[scale=.55,clip=true,angle=0]{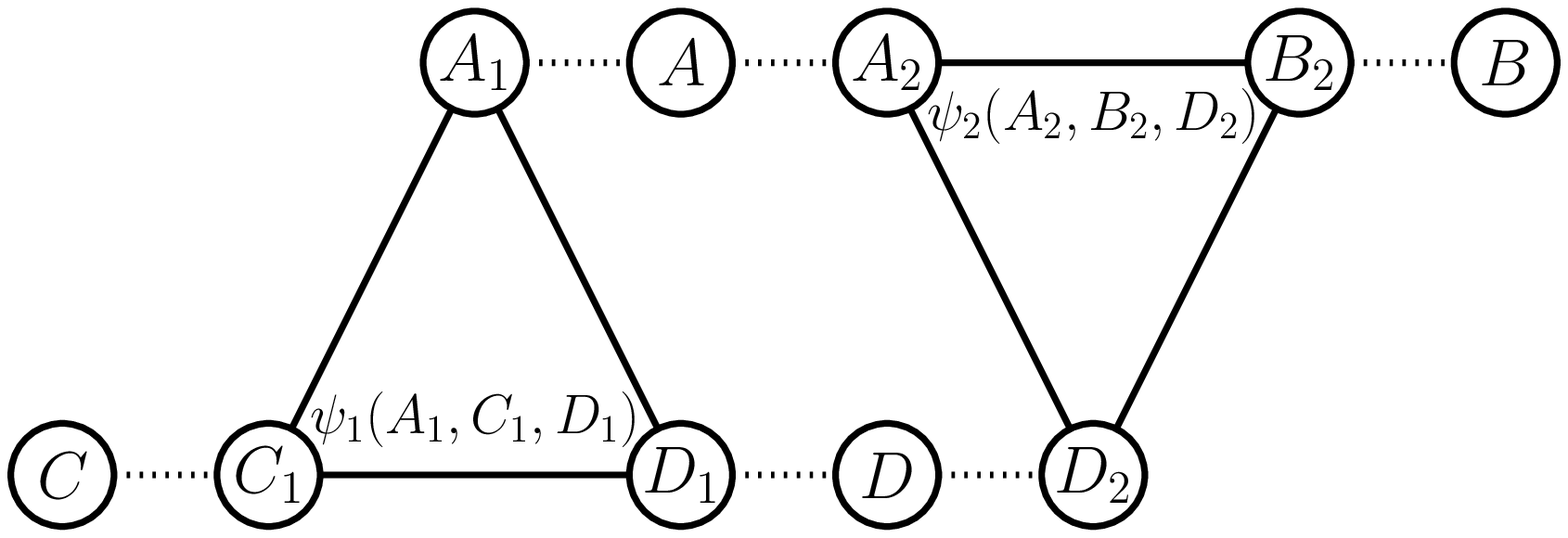}
 \caption{On the left, is an MRF with two factors,
   \(\psi_1(A,C,D)\) and \(\psi_2(A,B,D)\).  On the right, is
   the MRF found by cloning all variables, and then relaxing the 6
   resulting equivalence constraints (indicated by dashed lines).
   Besides the two original factors, now over cloned variables,
   \(\psi_1(A_1,C_1,D_1)\) and
   \(\psi_2(A_2,B_2,D_2)\), we now have twelve
   compensating factors, two each for the six equivalence constraint
   relaxed: one factor \(\theta(X)\) at each of the six cloned
   variables \(X\), and one factor \(\theta(X_i)\) each for variables
   \(B\) and \(C\) (involved in one equivalence constraint), and
   two factors \(\theta(X_i)\) each for variables \(A\) and \(D\)
   (involved in two equivalence constraints).}
 \label{fig:decomposition}
\end{center}
\end{figure}

Now, consider an equivalence constraint \(\eqv(X,X_i)\) in our
compensated MRF \(\gmp\).  Since the MRF is disconnected, the factor
\(\theta_{X_i}(X)\) interacts only with the compensating factors over
variable \(X\).  Similarly, the factor \(\theta_X(X_i)\) interacts
only with the factor \(\psi_i(\X^c_i)\), and the other
compensating factors over the other clone variables in
\(\X^c_i\).  Thus, our partial derivatives have the following
form:
\begin{align*}
\frac{\partial Z^\prime}{\partial \theta_{X_i}(x)}
& \propto \prod_{\substack{j : X \in \X_j \\ j \ne i}} \theta_{X_j}(x) \\
\frac{\partial Z^\prime}{\partial \theta_{X}(x_i)}
& \propto \sum_{\x^c_i \models x_i} \psi(\x^c_i) \cdot
    \prod_{\substack{Y_i \in \X^c_i:\\ Y \ne X}}
    \theta_Y(y_i) 
\end{align*}
Note again that we can compute the partial derivatives
\(\frac{\partial Z^\prime}{\partial \theta_{X_i}(x)}\) by
\(\frac{Z^\prime(x)}{\theta_{X_i}(x)}\), when \(\theta_{X_i}(x)\) is
positive.

For the MPE problem, we are interested in computing
\(\frac{\mpe^\prime(x)}{\theta_{X_i}(x)}\), which has a form analogous to
the above, except with maximizations instead of summations.  Moreover,
\(\frac{\mpe^\prime(x)}{\theta_{X_i}(x)}\) is independent of the parameter
\(\theta_{X_i}(x)\) (after taking into account the division).
The resulting fixed-point updates for the log parameters, are now:
\begin{align*}
\log \theta_{X_i}(x) 
& = - \frac{1}{2} \frac{\mpe^\prime(x)}{\theta_{X_i}(x)}
    + \frac{1}{2} \frac{\mpe^\prime(x_i)}{\theta_{X}(x_i)} \\
& = - \frac{1}{2} 
\sum_{\substack{j : X \in \X_j\\j \ne i}} 
\log \theta_{X_j}(x)
 + \frac{1}{2}
  \max_{\x^c_i \models x_i} \left[ \log \psi_i(\x^c_i) + 
    \sum_{\substack{Y_i \in \X^c_i: \\ Y \ne X}}
    \log \theta_{Y}(y_i) \right] + \log \alpha\\
& = - \frac{1}{2} 
\sum_{\substack{j : X \in \X_j\\ j \ne i}} 
\log \theta_{X_j}(x)
 + \frac{1}{2}
  \max_{\x^c_i \models x_i} \left[ \log \psi_i(\x^c_i) - 
    \sum_{\substack{Y_i \in \X^c_i : \\ Y \ne X}}
    \log \theta_{Y_i}(y) \right] + \log \alpha
\end{align*}
where we substitute \(\log \theta_Y(y_i)\) with
\(- \log \theta_{Y_i}(y)\), from our upper-bound
condition.  Here, \(\alpha \) can be treated as a normalizing constant,
which we can ignore, since it is canceled out in the joint
distribution of the compensated MRF.  We thus arrive at the block
coordinate descent update of the max-sum diffusion algorithm, as in
\cite[Equation 1.17]{SontagGJ11}.
\end{proof}

\begin{theorem} 
For a single equivalence constraint \(\eqv(X,X_i)\), with binary variables \(X\) and \(X_i\), 
a compensation scheme satisfies pr-dd iff the compensated model satisfies
\begin{equation} 
\frac{\theta(x)}{\theta(\n(x))}
= \left( 
\frac{Z^\prime(\n(x),x_i)/\theta(\n(x))\theta(x_i)}
     {Z^\prime(x,\n(x)_i)/\theta(x)\theta(\n(x)_i)}
\right)^{\frac{1}{2}}
\qquad\qquad
\frac{\theta(x_i)}{\theta(\n(x)_i)}
= \left( 
\frac{Z^\prime(x,\n(x)_i)/\theta(x)\theta(\n(x)_i)}
     {Z^\prime(\n(x),x_i)/\theta(\n(x))\theta(x_i)}
\right)^{\frac{1}{2}}
\end{equation}
The scheme satisfies mpe-dd iff it satisfies the above condition with \(\mpe^\prime(.)\) substituted for \(Z^\prime(.)\).
\end{theorem}

\begin{proof}
First, note that:
\begin{equation}
\frac{Z^\prime(x,x_i)}{\theta(x)\theta(x_i)}
= 
\frac{\partial^2 Z^\prime}{\partial \theta(x) \partial \theta(x_i)}
\end{equation}
which is a quantity that is independent of both of the unit factors
\(\theta(x)\) and \(\theta(x_i)\), since the partition function
\(Z^\prime\) is linear in \(\theta(x)\), and linear in \(\theta(x_i)\).

For binary variables \(X\) and \(X_i\) we have
\begin{eqnarray*}
\pr^\prime(x)
& = & \pr^\prime(x,x_i) + \pr^\prime(x,\n(x)_i)
= \frac{1}{Z^\prime} \frac{\partial^2 Z^\prime}{\partial \theta(x) \partial \theta(x_i)} \theta(x) \theta(x_i)  +
\frac{1}{Z^\prime} \frac{\partial^2 Z^\prime}{\partial \theta(x) \partial \theta(\n(x)_i)} \theta(x) \theta(\n(x)_i) \\
\pr^\prime(x_i)
& = & \pr^\prime(x,x_i) + \pr^\prime(\n(x),x_i)
= \frac{1}{Z^\prime} \frac{\partial^2 Z^\prime}{\partial \theta(x) \partial \theta(x_i)} \theta(x) \theta(x_i) +
\frac{1}{Z^\prime} \frac{\partial^2 Z^\prime}{\partial \theta(\n(x)) \partial \theta(x_i)} \theta(\n(x)) \theta(x_i)
\end{eqnarray*}
After substituting \(\theta(x_i) = \frac{1}{\theta(x)}\) (from our
upper-bounds condition), equating the above marginals, we get the
desired result after some rearranging.
\end{proof}

\end{document}